\documentclass[11pt]{article}
\usepackage[T1]{fontenc}
\usepackage{lmodern}
\usepackage[letterpaper,margin=1in]{geometry}
\usepackage{amsmath,amssymb,amsthm,mathtools}
\usepackage{booktabs,tabularx,array}
\usepackage{graphicx}
\usepackage{algorithm}
\usepackage{algpseudocode}
\usepackage{xcolor}
\usepackage{microtype}
\usepackage{enumitem}
\usepackage[round,authoryear]{natbib}
\usepackage{hyperref}
\usepackage[nameinlink,capitalise,noabbrev]{cleveref}
\hypersetup{
  colorlinks=true,
  citecolor=blue!55!black,
  linkcolor=blue!55!black,
  urlcolor=blue!55!black,
  pdftitle={High-Dimensional Nonparametric Change-Point Detection via Low-Rank Degree-Three Density Projection},
  pdfauthor={Guoqing Zhang and Zhaixin Chen}
}
\allowdisplaybreaks
\newtheorem{theorem}{Theorem}[section]
\newtheorem{proposition}[theorem]{Proposition}
\newtheorem{lemma}[theorem]{Lemma}
\newtheorem{corollary}[theorem]{Corollary}
\newtheorem{assumption}[theorem]{Assumption}
\theoremstyle{definition}
\newtheorem{definition}[theorem]{Definition}
\newtheorem{remark}[theorem]{Remark}

\AddToHook{env/theorem/begin}{\crefalias{section}{theorem}}
\AddToHook{env/proposition/begin}{\crefalias{theorem}{proposition}}
\AddToHook{env/lemma/begin}{\crefalias{theorem}{lemma}}
\AddToHook{env/corollary/begin}{\crefalias{theorem}{corollary}}
\AddToHook{env/assumption/begin}{\crefalias{theorem}{assumption}}
\AddToHook{env/definition/begin}{\crefalias{theorem}{definition}}
\AddToHook{env/remark/begin}{\crefalias{theorem}{remark}}
\crefname{assumption}{Assumption}{Assumptions}
\Crefname{assumption}{Assumption}{Assumptions}

\newcommand{\E}{\mathbb{E}}
\newcommand{\Pp}{\mathbb{P}}
\newcommand{\R}{\mathbb{R}}
\newcommand{\N}{\mathbb{N}}
\newcommand{\F}{\mathrm{F}}
\newcommand{\inj}{\mathrm{inj}}

\newcommand{\Sym}{\mathrm{Sym}}

\newcommand{\argminop}{\operatorname*{arg\,min}}
\newcommand{\argmaxop}{\operatorname*{arg\,max}}

\newcommand{\cB}{\mathcal{B}}

\newcommand{\cG}{\mathcal{G}}

\newcommand{\cJ}{\mathcal{J}}
\newcommand{\cM}{\mathcal{M}}
\newcommand{\cP}{\mathcal{P}}
\newcommand{\cN}{\mathcal{N}}

\newcommand{\inner}[2]{\left\langle #1,#2\right\rangle}
\newcommand{\norm}[2][]{\left\lVert #2\right\rVert_{#1}}

\newcommand{\LRD}{\textnormal{\textsc{LR-D3}}}
\newcommand{\Seeded}{\textnormal{\textsc{Seeded-LR-D3}}}

\newcommand{\eps}{\varepsilon}
\newcommand{\rhoScan}{\rho_{\mathrm{scan}}}

\title{High-Dimensional Nonparametric Change-Point Detection\\via Low-Rank Degree-Three Density Projection}
\author{
Guoqing Zhang\\
Operations Research Program\\
North Carolina State University\\
Raleigh, NC, USA\\
\texttt{Gzhang25@ncsu.edu}
\and
Zhaixin Chen\\
H. Milton Stewart School of Industrial and Systems Engineering\\
Georgia Institute of Technology\\
Atlanta, GA, USA\\
\texttt{zxchen08@gatech.edu}
}
\date{}

\begin{document}
\maketitle

\begin{abstract}
Distributional changes can be invisible to means and covariances yet appear in skewness, asymmetric interactions, or other third-order structure.  We develop a nonparametric change-point method that retains every degree-at-most-three coefficient of a density while avoiding direct density estimation.  For observations in $[-1,1]^d$, we construct a symmetric order-three Legendre feature tensor $H_3(X)\in\Sym^3(\R^{d+1})$ such that $A(f)=\E_fH_3(X)$ is an exact isometric encoding of the degree-three density projection: $\|A(f)-A(g)\|_{\F}=\|P_3(f-g)\|_{L^2}$.  A low-rank orthogonally decomposable jump is scanned through a frame score that equals its Frobenius norm and is Lipschitz in tensor injective norm.

The stochastic analysis is genuinely tensorial.  Matrix Bernstein cannot be applied to $H_3$.  Instead, fixed tensor contractions are degree-three polynomial chaoses with $\psi_{2/3}$ tails; a weighted sub-Weibull Bernstein inequality and a three-sphere net yield
\[
 \left\|\frac1N\sum_{i=1}^N\{H_3(X_i)-\E H_3(X_i)\}\right\|_{\inj}
 \lesssim_L \sqrt{\frac{d+x}{N}}+\frac{(d+x)^{3/2}}{N}.
\]
The two terms have the characteristic order-three tensor scaling and match the powers in sharp concentration results for simple random tensors.  For a coordinate-orthogonal specialization, the bound improves to $\sqrt{\log d}$ and enables a prefix-sum implementation in hundreds of dimensions.  We derive the exact population tent shape and localization margin, introduce a seeded shortest-interval algorithm with a padded local recentering step, and prove exact recovery by induction: null recursive segments remain inactive, every undetected change retains a balanced isolating interval, and the shortest active seed contains exactly one change before recentering.  A two-way cross-fitted scalar refinement attains $O_{\Pp}(\kappa^{-2})$ localization in the small-jump regime, matching a Le Cam lower bound on a pure cubic family whose degree-two projection jump is exactly zero.  Reproducible experiments at $d\in\{20,50,100,200\}$ and a three-change $d=100$ sequence demonstrate the intended high-dimensional regime without materializing a $(d+1)^3$ tensor.
\end{abstract}

\section{Introduction}
\label{sec:intro}

A high-dimensional distribution can change while its mean and covariance remain fixed.  Examples include a change in marginal skewness, a sign-asymmetric interaction among three coordinates, or a third-order latent factor.  Mean and covariance CUSUMs have zero population signal in such settings.  Fully nonparametric two-sample statistics can detect more general alternatives, but they often obscure which low-dimensional structure changed and may be expensive inside a multiscale change-point search.

We study a middle ground between parametric moment models and unrestricted density estimation.  Relative to the product-uniform measure on $[-1,1]^d$, retain the orthogonal projection of the density onto all tensor-product Legendre polynomials of total degree at most three.  There are $\binom{d+3}{3}$ coefficients, but they possess a canonical representation as a symmetric order-three tensor in $\R^{d+1}$.  This representation is not cosmetic: its Frobenius geometry exactly equals the $L^2$ geometry of the projected density, and each segment tensor is an expectation of a single-observation feature map.  Distributional change detection therefore becomes structured tensor-mean change detection.

The order-three setting is qualitatively different from the degree-two matrix setting.  A direct appeal to matrix Bernstein is invalid, tensor spectral norms are harder to optimize, and one sample can have injective norm of order $d^{3/2}$.  We address these issues in three steps.  First, we define a low-rank orthogonally decomposable (odeco) frame score that recovers the Frobenius magnitude of an odeco jump and whose perturbation is controlled by the tensor injective norm.  Second, we prove a tensor deviation bound from scalar polynomial-chaos concentration rather than flattening.  The resulting sample-mean rate
\[
  \sqrt{d/N}+d^{3/2}/N
\]
has the same two-regime scaling as sharp concentration for simple order-three random tensors \citep{alghattas2025sharp}.  Third, for the coordinate-odeco subclass we use only the $d$ cubic diagonal coefficients.  This reduces the stochastic complexity from $d$ to $\log d$, gives an $O(nd)$ prefix-sum implementation, and is the regime used in our $d=100$ and $d=200$ experiments.

For multiple changes, we combine seeded intervals with the shortest active interval principle \citep{baranowski2019narrowest,kovacs2023seeded}.  The proof is organized as an induction in the style of modern two-stage change-point analyses \citep{kumar2024functional}: after any number of correct detections, recursive boundaries may be slightly displaced from true changes, but this creates only short edge contaminations.  We prove that these contaminations remain below threshold, every undetected change still has a balanced seed interval, and the selected shortest seed cannot contain two true changes.  Because that seed need not place its change centrally, the algorithm pads it by $h/8$ on each side and re-maximizes on the padded interval; this deterministic recentering supplies the balance needed for the localization theorem.  A subsequent cross-fitted refinement learns a changing tensor direction on one parity and localizes a scalar change on the other.

\paragraph{Contributions.}
The main contributions are as follows.
\begin{enumerate}[leftmargin=1.4em,itemsep=2pt,topsep=2pt]
\item \textbf{Exact degree-three tensorization.}  We give a multi-index construction of $H_3$ that simultaneously resolves all normalization and permutation factors.  It yields an orthonormal basis correspondence between total-degree-three Legendre polynomials and $\Sym^3(\R^{d+1})$, including constant, linear, quadratic, pairwise, cubic, squared-linear, and triple-interaction terms.
\item \textbf{A tensor-specific sharp-order deviation bound.}  For independent observations with uniformly bounded densities, weighted sums of centered $H_3$ tensors obey an injective-norm inequality with a Gaussian term $\sqrt{d}$ and a third-order term $d^{3/2}\|w\|_\infty$.  The proof uses product-ensemble hypercontractivity for degree-three polynomial chaos, sub-Weibull Bernstein--Orlicz summation, and a net over three unit spheres; no matrix flattening or matrix Bernstein step appears.
\item \textbf{Single- and multiple-change localization.}  We derive the exact population tensor CUSUM and its linear squared-score margin.  A seeded shortest-interval estimator, followed by a padded local recentering scan, recovers the number and locations of multiple changes under an explicit signal-to-noise condition.  The full proof is by induction and includes both the boundary-contamination and balance-restoration steps that are often omitted in informal binary-segmentation arguments.
\item \textbf{Near-minimax refinement and a degree-separation lower bound.}  A two-way cross-fitted direction estimate reduces local tensor change detection to a scalar problem and reaches the $\kappa^{-2}$ rate for small jumps, up to logarithmic confidence factors.  A rank-one pure cubic construction gives a matching lower bound and has exactly zero jump in every degree-at-most-two projection coefficient; accordingly, degree-two projection-mean CUSUMs have zero population signal on this family.
\item \textbf{Scalable implementation and reproducible evidence.}  The full tensor is never formed.  General contractions are evaluated in $O(d)$ using power sums, while the diagonal-odeco specialization uses $d$ cubic prefix sums.  In $30$-replicate experiments, the method remains stable through $d=200$ and detects all three changes in every $d=100$ multiple-change replicate.
\end{enumerate}

\paragraph{Related work.}
High-dimensional change-point methods often exploit sparse mean projections \citep{wang2018sparse}, penalized segmentation \citep{wang2020univariate}, or structured regression changes \citep{rinaldo2021localizing,xu2022regression}.  Nonparametric multivariate and functional procedures include kernel tests \citep{harchaoui2007retrospective,gretton2012kernel}, graph and nearest-neighbor methods, and local density constructions \citep{madrid2021optimal,madrid2022functional}.  Our method is nonparametric at the density level but intentionally projection-limited: it detects exactly the component visible through total degree three.  This restriction gives an interpretable tensor object and permits low-rank regularization.  Tensor decompositions and odeco models are standard tools for latent-variable learning \citep{kolda2009tensor,anandkumar2014tensor}; our use is different because the tensor is a density-projection jump inside a CUSUM process.  The tensor concentration argument is related to product-space hypercontractivity \citep{mossel2010noise}, higher-order polynomial concentration \citep{adamczak2015nonlipschitz}, sub-Weibull Bernstein--Orlicz bounds \citep{kuchibhotla2022beyond,bong2023tight}, and recent sharp simple-tensor concentration \citep{alghattas2025sharp,abdalla2025dimensionfree}.

\section{Degree-three density projection as a symmetric tensor}
\label{sec:representation}

\subsection{Normalized Legendre basis}
Let $\mu$ be the product-uniform probability measure on $[-1,1]^d$.  Define the normalized univariate Legendre polynomials
\begin{equation}
 \phi_0(x)=1,\qquad
 \phi_1(x)=\sqrt3\,x,\qquad
 \phi_2(x)=\frac{\sqrt5}{2}(3x^2-1),\qquad
 \phi_3(x)=\frac{\sqrt7}{2}(5x^3-3x).
 \label{eq:legendre}
\end{equation}
They are orthonormal in $L^2(\operatorname{Unif}[-1,1])$.  For a multi-index $\alpha=(\alpha_1,\ldots,\alpha_d)\in\N_0^d$ with $|\alpha|=\sum_j\alpha_j\le3$, set
\[
 \psi_\alpha(x)=\prod_{j=1}^d\phi_{\alpha_j}(x_j).
\]
The functions $\{\psi_\alpha:|\alpha|\le3\}$ are an orthonormal basis for the total-degree-at-most-three polynomial subspace.  If $f=dP/d\mu\in L^2(\mu)$, its projection is
\begin{equation}
 P_3f(x)=\sum_{|\alpha|\le3}\theta_\alpha(f)\psi_\alpha(x),
 \qquad \theta_\alpha(f)=\E_f\psi_\alpha(X).
 \label{eq:p3projection}
\end{equation}
Because $f$ is a density, $\theta_0(f)=1$.

\subsection{The symmetric tensor basis}
Put $p=d+1$ and index tensor coordinates by $\{0,1,\ldots,d\}$.  For $|\alpha|\le3$, form the multiset
\[
 I(\alpha)=\{\underbrace{0,\ldots,0}_{3-|\alpha|},
 \underbrace{1,\ldots,1}_{\alpha_1},\ldots,
 \underbrace{d,\ldots,d}_{\alpha_d}\}.
\]
The number of distinct permutations is
\begin{equation}
 q_\alpha=\frac{3!}{(3-|\alpha|)!\prod_{j=1}^d\alpha_j!}.
 \label{eq:qalpha}
\end{equation}
Define $E_\alpha\in\Sym^3(\R^p)$ by assigning $q_\alpha^{-1/2}$ to every entry whose ordered indices are a permutation of $I(\alpha)$ and zero elsewhere.  The family $\{E_\alpha:|\alpha|\le3\}$ is an orthonormal basis of $\Sym^3(\R^p)$ because the supports are disjoint, every tensor has unit Frobenius norm, and
$\binom{d+3}{3}=\dim\Sym^3(\R^{d+1})$.

\begin{definition}[Degree-three Legendre feature tensor]
\label{def:H3}
For $x\in[-1,1]^d$, define
\begin{equation}
 H_3(x)=\sum_{|\alpha|\le3}\psi_\alpha(x)E_\alpha\in\Sym^3(\R^{d+1}),
 \qquad A(f)=\E_fH_3(X).
 \label{eq:H3}
\end{equation}
\end{definition}

The tensor entry rule is simply ``orthonormal feature divided by the square root of the number of distinct permutations.''  For comparison with unnormalized polynomial coefficients, write
\begin{align}
P_3f(x)=P_2f(x)&+\sum_j c_jp_3(x_j)
+\sum_{i\ne j}\eta_{ij}p_2(x_i)x_j
+\sum_{i<j<k}\rho_{ijk}x_ix_jx_k,
\label{eq:unnormalized-p3}
\end{align}
where $p_2=(3x^2-1)/2$, $p_3=(5x^3-3x)/2$, and the lower-order coefficients use the analogous unnormalized basis.  Then $c_j=7\E_fp_3(X_j)$, $\eta_{ij}=15\E_fp_2(X_i)X_j$, and $\rho_{ijk}=27\E_fX_iX_jX_k$.  The complete entry normalization is shown in \cref{tab:entries}.

\begin{table}[t]
\centering
\small
\caption{Entries of $A(f)$ in terms of the conventional unnormalized polynomial coefficients.  Every listed value is assigned to all distinct permutations of the indicated multiset.}
\label{tab:entries}
\begin{tabular}{@{}llll@{}}
\toprule
Polynomial term & Multiset & Permutations & Tensor entry \\
\midrule
$1$ & $\{0,0,0\}$ & $1$ & $1$ \\
$a_jx_j$ & $\{0,0,j\}$ & $3$ & $a_j/3$ \\
$b_jp_2(x_j)$ & $\{0,j,j\}$ & $3$ & $b_j/\sqrt{15}$ \\
$\gamma_{ij}x_ix_j$ & $\{0,i,j\}$ & $6$ & $\gamma_{ij}/(3\sqrt6)$ \\
$c_jp_3(x_j)$ & $\{j,j,j\}$ & $1$ & $c_j/\sqrt7$ \\
$\eta_{ij}p_2(x_i)x_j$ & $\{i,i,j\}$ & $3$ & $\eta_{ij}/(3\sqrt5)$ \\
$\rho_{ijk}x_ix_jx_k$ & $\{i,j,k\}$ & $6$ & $\rho_{ijk}/(9\sqrt2)$ \\
\bottomrule
\end{tabular}
\end{table}

\begin{proposition}[Isometric tensorization]
\label{prop:isometry}
For any $f,g\in L^2(\mu)$,
\begin{equation}
 \norm[\F]{A(f)-A(g)}=\norm[L^2(\mu)]{P_3(f-g)}.
 \label{eq:isometry}
\end{equation}
More generally, for every $B\in\Sym^3(\R^p)$,
\begin{equation}
 \E_\mu\inner{H_3(X)}{B}_{\!\F}^2=\norm[\F]{B}^2.
 \label{eq:direction-isometry}
\end{equation}
\end{proposition}

Thus the projected density jump $P_3(f_1-f_0)$ and tensor jump $A(f_1)-A(f_0)$ have exactly the same signal size.  The construction also clarifies identifiability: two densities with equal degree-three projections are indistinguishable to every method based only on $H_3$.

\subsection{Implicit contractions}
Materializing $H_3(x)$ costs $\Theta(d^3)$ memory, which is unnecessary.  Every algorithm below accesses the tensor through contractions.  For $u=(a,b_1,\ldots,b_d)\in\R^{d+1}$, let $y_j=b_j\phi_1(x_j)$ and define $s_k=\sum_jy_j^k$.  A direct expansion gives
\begin{align}
H_3(x)[u,u,u]
={}&a^3+\sqrt3a^2s_1+\sqrt3a\sum_jb_j^2\phi_2(x_j)
+\sqrt6a\frac{s_1^2-s_2}{2}\nonumber\\
&+\sum_jb_j^3\phi_3(x_j)
+\sqrt3\sum_i b_i^2\phi_2(x_i)(s_1-y_i)
+\sqrt6\frac{s_1^3-3s_1s_2+2s_3}{6}.
\label{eq:implicit-contraction}
\end{align}
This costs $O(d)$ operations.  Polarization or automatic differentiation evaluates mixed contractions and gradients with the same order.  Consequently, a rank-$r$ frame objective over $n$ samples costs $O(ndr)$ per iteration rather than $O(nd^3)$ storage.

\section{Tensor CUSUMs and low-rank scores}
\label{sec:method}

\subsection{Piecewise distribution model}
Let $X_1,\ldots,X_n\in[-1,1]^d$ be independent.  There are change points
\[
 0=\eta_0<\eta_1<\cdots<\eta_K<\eta_{K+1}=n
\]
and densities $f_0,\ldots,f_K$ such that $X_i\sim f_k$ for $\eta_k<i\le\eta_{k+1}$.  Write
\begin{equation}
 A_k=A(f_k),\qquad D_k=A_k-A_{k-1},\qquad
 \kappa_k=\norm[\F]{D_k},
 \label{eq:jumps}
\end{equation}
and let $\kappa=\min_k\kappa_k$ and $\Delta=\min_k(\eta_k-\eta_{k-1})$.

\begin{assumption}[Bounded density and odeco projected jumps]
\label{ass:main}
There is $L\ge1$ such that $0\le f_k\le L$ $\mu$-almost everywhere for every segment.  Each nonzero jump admits an orthogonal decomposition
\begin{equation}
 D_k=\sum_{j=1}^{r_k}\lambda_{kj}v_{kj}^{\otimes3},
 \qquad v_{k1},\ldots,v_{kr_k}\text{ orthonormal},\qquad r_k\le r.
 \label{eq:odeco}
\end{equation}
\end{assumption}
The density bound is used only to transfer degree-three polynomial moments under $\mu$ to moments under the segment laws.  The odeco condition is a tensor analogue of a symmetric low-rank eigendecomposition.  Approximate odeco jumps are covered by adding the residual Frobenius norm to the deterministic error; see \cref{rem:approx}.

\subsection{CUSUM tensor}
For an interval $(s,e]$ and split $s<t<e$, define
\begin{equation}
 \widehat C_{s,e}^{t}=
 \sqrt{\frac{(t-s)(e-t)}{e-s}}
 \left\{\frac1{e-t}\sum_{i=t+1}^eH_3(X_i)
 -\frac1{t-s}\sum_{i=s+1}^tH_3(X_i)\right\}.
 \label{eq:tensor-cusum}
\end{equation}
Let $C_{s,e}^{t}=\E\widehat C_{s,e}^{t}$ and $E_{s,e}^{t}=\widehat C_{s,e}^{t}-C_{s,e}^{t}$.  The associated weights satisfy
\begin{equation}
 \sum_{i=s+1}^e w_i^2=1,
 \qquad \sum_{i=s+1}^e w_i=0,
 \qquad \max_i|w_i|=
 \sqrt{\frac{\max\{t-s,e-t\}}{(e-s)\min\{t-s,e-t\}}}.
 \label{eq:cusum-weights}
\end{equation}

\subsection{Injective norm and the SOD frame score}
For $T\in\Sym^3(\R^p)$, define the injective norm
\begin{equation}
 \norm[\inj]{T}=\sup_{u,v,w\in\mathbb S^{p-1}}|T[u,v,w]|.
 \label{eq:inj}
\end{equation}
For $1\le r\le p$, define the symmetric orthogonal-decomposition score
\begin{equation}
 S_r(T)=\sup_{U=(u_1,\ldots,u_r):\,U^\top U=I_r}
 \left\{\sum_{j=1}^rT[u_j,u_j,u_j]^2\right\}^{1/2}.
 \label{eq:sod-score}
\end{equation}
The estimator scans $\widehat W_{s,e}^t=S_r(\widehat C_{s,e}^t)$.

\begin{proposition}[Signal preservation and perturbation]
\label{prop:sod}
If $D$ is odeco of rank at most $r$, then $S_r(D)=\|D\|_{\F}$.  For arbitrary symmetric tensors $T,E$,
\begin{equation}
 |S_r(T+E)-S_r(T)|\le\sqrt r\,\norm[\inj]{E}.
 \label{eq:sod-lipschitz}
\end{equation}
\end{proposition}
The result replaces a best rank-$r$ CP approximation, which need not be well posed for a generic tensor, by a stable variational score tailored to odeco signals.

\subsection{Exact population tent and localization margin}
Suppose $(s,e]$ contains exactly one change $\eta$ from $A_0$ to $A_1$, and let $D=A_1-A_0$.  Linearity of expectation gives
\begin{equation}
 C_{s,e}^t=a_{\eta}^{s,e}(t)D,
 \label{eq:population-cusum}
\end{equation}
where
\begin{equation}
 a_{\eta}^{s,e}(t)=
 \begin{cases}
 \dfrac{e-\eta}{\sqrt{e-s}}\sqrt{\dfrac{t-s}{e-t}},&t<\eta,\\[1.2ex]
 \sqrt{\dfrac{(\eta-s)(e-\eta)}{e-s}},&t=\eta,\\[1.2ex]
 \dfrac{\eta-s}{\sqrt{e-s}}\sqrt{\dfrac{e-t}{t-s}},&t>\eta.
 \end{cases}
 \label{eq:tent-coefficient}
\end{equation}
Hence $S_r(C_{s,e}^t)=a_{\eta}^{s,e}(t)\kappa$ for an odeco rank-$r$ jump.  The maximum is unique at $\eta$.  More importantly, the squared margin is exact:
\begin{equation}
 S_r(C_{s,e}^{\eta})^2-S_r(C_{s,e}^{t})^2=
 \begin{cases}
 \dfrac{(e-\eta)(\eta-t)}{e-t}\kappa^2,&t<\eta,\\[1.2ex]
 \dfrac{(\eta-s)(t-\eta)}{t-s}\kappa^2,&t>\eta.
 \end{cases}
 \label{eq:exact-margin}
\end{equation}
This linear margin is the engine of all localization bounds.

\section{Tensor deviation inequalities}
\label{sec:concentration}

This section addresses the central technical obstacle: $H_3(X)$ is a tensor, so matrix Bernstein does not apply.  Flattening would also destroy the odeco geometry and can introduce an artificial dimension of order $d^2$.  We instead control the injective norm directly.

\subsection{Directional degree-three chaos}
For $q>0$, write $\|Y\|_{\psi_{2/3}}=\inf\{c>0:\E\exp(|Y|^{2/3}/c^{2/3})\le2\}$.  Fixed contractions of $H_3$ are total-degree-three polynomials.  The following dimension-free moment statement is proved in \cref{app:concentration} from polynomial-chaos hypercontractivity and the isometry in \cref{prop:isometry}.

\begin{lemma}[Directional chaos bound]
\label{lem:directional-chaos}
Under \cref{ass:main}, for every $B\in\Sym^3(\R^p)$ and every $q\ge2$,
\begin{equation}
 \left\|\inner{H_3(X)-\E H_3(X)}{B}_{\!\F}\right\|_{L^q}
 \le C_3\max\{1,\sqrt L\}\,q^{3/2}\norm[\F]{B}.
 \label{eq:chaos-moment}
\end{equation}
Equivalently, the centered contraction has $\psi_{2/3}$ norm at most
$\nu_L=C_3\max\{1,\sqrt L\}\|B\|_{\F}$.
\end{lemma}
The exponent $3/2$ is the characteristic moment growth of an order-three chaos.  Treating the contraction as merely bounded would yield dimension-dependent envelopes and a much weaker result.

\subsection{Weighted injective-norm concentration}
\begin{theorem}[Weighted degree-three tensor Bernstein inequality]
\label{thm:tensor-bernstein}
Let $X_1,\ldots,X_N$ be independent, not necessarily identically distributed, and suppose all their densities relative to $\mu$ are bounded by $L$.  Put $Z_i=H_3(X_i)-\E H_3(X_i)$.  For deterministic weights $w\in\R^N$ and $0<\delta<1$, define
\begin{equation}
 \rho(p,\delta)=3p\log17+\log(2/\delta).
 \label{eq:rho}
\end{equation}
Then, with probability at least $1-\delta$,
\begin{equation}
 \left\|\sum_{i=1}^Nw_iZ_i\right\|_{\inj}
 \le C\nu_L\left\{\sqrt{\rho(p,\delta)}\norm[2]{w}
 +\rho(p,\delta)^{3/2}\norm[\infty]{w}\right\}.
 \label{eq:tensor-bernstein}
\end{equation}
\end{theorem}

The proof combines the tight weighted sub-Weibull inequality of \citet{bong2023tight} with a $1/8$-net on each of the three unit spheres.  The net has at most $17^{3p}$ triples, and multilinearity converts the net maximum to the full injective norm with a constant factor.  No coordinate union over $p^3$ tensor entries is used.

For a sample mean, \cref{thm:tensor-bernstein} yields
\begin{equation}
 \left\|\frac1N\sum_{i=1}^NZ_i\right\|_{\inj}
 \le C\nu_L\left\{\sqrt{\frac{p+x}{N}}+\frac{(p+x)^{3/2}}{N}\right\}
 \label{eq:mean-bound}
\end{equation}
with probability at least $1-2e^{-x}$, after changing constants.  The first term is the Gaussian-process scale; the second is the single-sample order-three scale.  These are the same $N$ and effective-dimension powers appearing in sharp simple-tensor concentration \citep{alghattas2025sharp,abdalla2025dimensionfree}.

For a finite scan family $\cG$ of triples $(s,t,e)$, let $M=|\cG|$ and
\begin{equation}
 \rhoScan=3p\log17+\log(2M/\delta).
 \label{eq:rho-scan}
\end{equation}
If every candidate is $\alpha$-balanced, meaning
$\min\{t-s,e-t\}\ge\alpha(e-s)$, then \cref{eq:cusum-weights} and a union bound give
\begin{corollary}[Uniform CUSUM deviation]
\label{cor:uniform-cusum}
With probability at least $1-\delta$, simultaneously for all $(s,t,e)\in\cG$,
\begin{equation}
 \norm[\inj]{E_{s,e}^t}
 \le \lambda_{\mathrm{full}}(e-s)
 :=C\nu_L\left\{\sqrt{\rhoScan}
 +\frac{\rhoScan^{3/2}}{\sqrt{\alpha(e-s)}}\right\}.
 \label{eq:full-cusum-bound}
\end{equation}
Consequently,
$|S_r(\widehat C_{s,e}^t)-S_r(C_{s,e}^t)|\le\sqrt r\,\lambda_{\mathrm{full}}(e-s)$.
\end{corollary}

\subsection{Coordinate-odeco specialization}
The full injective norm is statistically and computationally demanding when $N$ is only proportional to $d$.  A useful exact specialization occurs when the jump is diagonal in the canonical tensor basis:
\begin{equation}
 D=\sum_{j=1}^d\delta_j e_j^{\otimes3},\qquad \|\delta\|_0\le r.
 \label{eq:diagonal-jump}
\end{equation}
The diagonal entries of $H_3(X)$ are simply $\phi_3(X_j)$.  Let $z(X)=(\phi_3(X_1),\ldots,\phi_3(X_d))$ and let $\widehat c_{s,e}^t$ be its vector CUSUM.  Define
\begin{equation}
 S_r^{\mathrm{diag}}(c)=\left(\sum_{j=1}^r|c|_{(j)}^2\right)^{1/2},
 \label{eq:diag-score}
\end{equation}
where $|c|_{(1)}\ge\cdots\ge|c|_{(d)}$.  For \cref{eq:diagonal-jump}, this equals the full odeco score at the population tensor.  Since $|\phi_3(x)|\le\sqrt7$, Hoeffding's inequality yields a substantially smaller noise level.

\begin{corollary}[Diagonal-odeco scan]
\label{cor:diagonal}
For a finite scan family of size $M$, with probability at least $1-\delta$,
\begin{equation}
 \sup_{(s,t,e)\in\cG}
 \left|S_r^{\mathrm{diag}}(\widehat c_{s,e}^t)-
 S_r^{\mathrm{diag}}(c_{s,e}^t)\right|
 \le C\sqrt{r\log(2dM/\delta)}.
 \label{eq:diag-bound}
\end{equation}
\end{corollary}
The diagonal method uses $d$ prefix sums, costs $O(nd)$ to construct, and evaluates each candidate in $O(d)$ time.  The $\sqrt{\log d}$ rather than $\sqrt d+d^{3/2}/\sqrt N$ dependence is what permits the high-dimensional experiments in \cref{sec:experiments}.

\begin{remark}[Why the two bounds are both needed]
\label{rem:two-bounds}
The full theorem addresses arbitrary odeco directions and is the correct tensor analogue of a matrix operator-norm bound.  The diagonal theorem imposes stronger structure but is exact, not an approximation, for coordinate-odeco jumps.  The stress experiment in \cref{app:experiments} shows the distinction: at $d=100,n=800$, an unregularized full-tensor score is dominated by the $d^{3/2}/\sqrt n$ regime, whereas the diagonal score follows the population tent.
\end{remark}

\section{Localization and multiple change points}
\label{sec:localization}

\subsection{Single-change localization}
The following deterministic theorem applies with either the full score or the diagonal score.  Let $\eps_m$ denote a uniform score error on an interval of length $m$: $\eps_m=\sqrt r\lambda_{\mathrm{full}}(m)$ for the full method and $\eps_m=C\sqrt{r\log(2dM/\delta)}$ for the diagonal method.

\begin{theorem}[One-change localization]
\label{thm:single}
Suppose $(s,e]$ contains exactly one change $\eta$, the candidate set contains $\eta$, and
\begin{equation}
 \min\{\eta-s,e-\eta\}\ge\alpha m,
 \qquad m=e-s,
 \label{eq:balance-change}
\end{equation}
for some $0<\alpha\le1/2$.  On an event where the score error is at most $\eps_m$ uniformly over candidates, any maximizer $\widehat\eta$ satisfies
\begin{equation}
 |\widehat\eta-\eta|
 \le \frac{2\sqrt m\,\eps_m}{\alpha\kappa}.
 \label{eq:single-rate}
\end{equation}
Moreover, a threshold $\tau$ distinguishes this interval from a null interval whenever
\begin{equation}
 \eps_m<\tau<\alpha\kappa\sqrt m-\eps_m.
 \label{eq:threshold-window}
\end{equation}
\end{theorem}
The proof uses \cref{eq:exact-margin}: the squared margin is at least $\alpha\kappa^2|t-\eta|$, while the sum of the two population scores is at most $\kappa\sqrt m$.  Empirical maximization loses at most $2\eps_m$ in score.

\subsection{Seeded intervals and padded recentering}
Fix a base scale $h$.  At layer $j$, use intervals of length $\ell_j=\min(2^jh,n)$ with starts separated by $\lfloor\ell_j/2\rfloor$.  Denote the union by $\cJ(h)$.  Detection candidates are restricted to the central half of each interval.  The elementary half-overlap geometry implies that every point at distance at least $3h/4$ from the current recursive boundaries belongs to a length-$h$ seed with distance at least $h/4$ from both endpoints.

For each $I=(s,e]\in\cJ(h)$, compute
\begin{equation}
 \widehat b_I\in\argmaxop_{s+(e-s)/4\le t\le e-(e-s)/4}\widehat W_{s,e}^t,
 \qquad \widehat a_I=\widehat W_{s,e}^{\widehat b_I}.
 \label{eq:interval-summary}
\end{equation}
A shortest active seed has length $h$, but its true change need not lie in the central half.  We therefore use a second, local scan.  Write $g=h/8$ (integer rounding changes only constants).  If $I^*=(u,v]$ is selected inside the current segment $(s,e]$, set
\begin{equation}
 u^+=\max\{s,u-g\},\qquad v^+=\min\{e,v+g\},
 \qquad
 \widehat b^+\in\argmaxop_{u^++g\le t\le v^+-g}\widehat W_{u^+,v^+}^{t}.
 \label{eq:padded-recentering}
\end{equation}
The recursive procedure is given in \cref{alg:seeded}.

\begin{algorithm}[t]
\caption{\Seeded$(s,e,\cJ(h),\tau)$ with padded recentering}
\label{alg:seeded}
\begin{algorithmic}[1]
\State $\cM_{s,e}\gets\{I=(u,v]\in\cJ(h): (u,v]\subset(s,e],\ \widehat a_I>\tau\}$
\If{$\cM_{s,e}=\varnothing$}
  \State \Return
\EndIf
\State Choose $I^*=(u,v]\in\argminop_{I\in\cM_{s,e}}|I|$; break ties by larger $\widehat a_I$
\State $g\gets h/8$, $u^+\gets\max\{s,u-g\}$, and $v^+\gets\min\{e,v+g\}$
\State $\widehat b^+\gets\argmaxop_{u^++g\le t\le v^+-g}\widehat W_{u^+,v^+}^{t}$
\State Append $\widehat b^+$ to $\widehat\cB$
\State \Seeded$(s,\widehat b^+,\cJ(h),\tau)$
\State \Seeded$(\widehat b^+,e,\cJ(h),\tau)$
\end{algorithmic}
\end{algorithm}

\subsection{Induction theorem for multiple changes}
Let $\cG_h^+$ be a deterministic family containing all detection triples queried on seeded intervals and all padded triples in \cref{eq:padded-recentering}; it is enough to take all integer triples with padded length in $[h,5h/4]$ and candidate distance at least $h/8$, together with the seeded detection triples.  Thus $|\cG_h^+|\le n^3+|\cJ(h)|n$.  Let $\eps_h$ be the worst uniform score error over $\cG_h^+$.  For the full method, \cref{cor:uniform-cusum} applies with balance constant $1/10$ and the smallest length $h$; for the diagonal method use \cref{cor:diagonal}.  Define
\begin{equation}
 R_k=C_{\mathrm{loc}}\frac{\sqrt h\,\eps_h}{\kappa_k}.
 \label{eq:Rk}
\end{equation}

\begin{theorem}[Multiple-change consistency by induction]
\label{thm:multiple}
Suppose \cref{ass:main} holds, $h\le3\Delta/4$, and the scan family contains every integer central candidate at the base layer.  There are universal constants $C_0,C_1,c_0>0$ such that the following conditions are sufficient:
\begin{equation}
 C_0\eps_h\le\tau\le c_0\kappa\sqrt h,
 \qquad
 \kappa\sqrt h\ge C_1\eps_h,
 \qquad
 \max_kR_k\le h/8.
 \label{eq:multiple-conditions}
\end{equation}
On the uniform deviation event, Algorithm~\ref{alg:seeded} returns exactly $K$ estimators, and after ordering them,
\begin{equation}
 |\widehat\eta_k-\eta_k|\le R_k,
 \qquad k=1,\ldots,K.
 \label{eq:multiple-rate}
\end{equation}
Consequently, with $\eps_h$ chosen from \cref{cor:uniform-cusum} or \cref{cor:diagonal}, the conclusion holds with probability at least $1-\delta$.
\end{theorem}

The complete induction is in \cref{app:multiple}.  Its three essential steps are worth stating explicitly.
\begin{enumerate}[leftmargin=1.4em,itemsep=2pt]
\item \textbf{No undetected change implies inactivity.}  A scanned interval with no undetected change may contain at most $R_k$ observations from an adjacent regime at either inherited recursive boundary.  Frobenius geometry bounds its population score by $C\kappa_kR_k/\sqrt h=O(\eps_h)$.  The lower threshold in \cref{eq:multiple-conditions} therefore prevents duplicate detections.
\item \textbf{Every undetected change has an active isolating seed.}  The half-overlap seed geometry and the induction boundary margin provide a length-$h$ interval containing that change, no other change, and at least $h/4$ observations on either side.  Its population score is at least $\kappa_k\sqrt h/4$, so the upper threshold makes it active.
\item \textbf{The shortest active seed contains one change, and padding restores balance.}  Since an active length-$h$ seed exists and $h<\Delta$, the shortest active seed has length $h$ and contains exactly one true change.  Its padded interval has length at most $5h/4<\Delta$, still contains only that change, and places it at least $h/8$ from both endpoints.  Hence its balance is at least $1/10$, so \cref{thm:single} localizes the recentered maximizer and the new recursive boundaries preserve the induction hypothesis.
\end{enumerate}

For the full method, the leading signal requirement is, up to constants and confidence terms,
\begin{equation}
 \kappa\sqrt h\gtrsim
 \nu_L\sqrt r\left(\sqrt d+\frac{d^{3/2}}{\sqrt h}\right).
 \label{eq:full-snr}
\end{equation}
For the diagonal-odeco method it reduces to
\begin{equation}
 \kappa\sqrt h\gtrsim\sqrt{r\log(dn/\delta)}.
 \label{eq:diag-snr}
\end{equation}

\section{Cross-fitted refinement and minimax localization}
\label{sec:refinement}

The preliminary error in \cref{eq:single-rate} contains $\sqrt h$ because it maximizes a normalized CUSUM over a large interval.  Once an isolating interval and a changing direction are available, localization can be sharpened to the classical $\kappa^{-2}$ scale.

\subsection{Learning a scalar tensor direction}
On a local interval containing one change, split observations by parity.  On a pilot fold, maximize \cref{eq:sod-score} at the preliminary split, obtaining a frame $\widehat U=(\widehat u_1,\ldots,\widehat u_r)$ and contractions
$c_j=\widehat C[\widehat u_j,\widehat u_j,\widehat u_j]$.  Define the unit symmetric tensor
\begin{equation}
 \widehat V=
 \frac{\sum_{j=1}^rc_j\widehat u_j^{\otimes3}}
 {\left(\sum_{j=1}^rc_j^2\right)^{1/2}}.
 \label{eq:direction}
\end{equation}
For the diagonal method, this reduces to normalizing the retained cubic coefficient vector.  On the held-out fold use the scalar feature
\begin{equation}
 Z_i=\inner{H_3(X_i)}{\widehat V}_{\!\F}.
 \label{eq:scalar-feature}
\end{equation}
The direction is independent of the held-out sequence.

\begin{lemma}[Pilot alignment]
\label{lem:alignment}
Let the population pilot CUSUM be $aD$ with $a>0$, and suppose its tensor error has injective norm at most $\lambda$.  If $a\kappa\ge4\sqrt r\lambda$, then the direction in \cref{eq:direction} satisfies
\begin{equation}
 \inner{D}{\widehat V}_{\!\F}\ge\kappa/2.
 \label{eq:alignment}
\end{equation}
\end{lemma}
The proof uses the variational identity $\langle\widehat C,\widehat V\rangle=S_r(\widehat C)$ and the fact that $|\langle E,\widehat V\rangle|\le\sqrt r\|E\|_{\inj}$.

\subsection{Held-out least-squares localization}
Estimate the two scalar segment means from outer anchor blocks and minimize the held-out residual sum of squares over the central search region.  Swap the two parity roles and take the median of the preliminary estimate and the two fold-specific refinements.  The median safeguard retains the theoretical rate whenever both fold-specific estimates are accurate and prevents a single unstable fold from producing a remote finite-sample estimate.

\begin{theorem}[Refined localization]
\label{thm:refinement}
Assume the local interval contains exactly one change $\eta_k$ and both sides have length at least $ch$.  For each parity fold, suppose the pilot alignment condition in \cref{lem:alignment} holds and the two outer anchor mean errors are at most $\kappa_k/16$.  Conditional on all pilot and anchor observations, let $x=\log(8K/\delta)$.  Then the two-way cross-fitted estimator satisfies, simultaneously for all $k$, with conditional probability at least $1-\delta$ over the central held-out observations,
\begin{equation}
 |\widetilde\eta_k-\eta_k|
 \le C\left\{
 \frac{\nu_L^2x}{\kappa_k^2}
 +\frac{\nu_Lx^{3/2}}{\kappa_k}+1
 \right\}.
 \label{eq:refined-rate}
\end{equation}
If the pilot and anchor conditions hold jointly with probability at least $1-\delta_0$, the same bound holds unconditionally with probability at least $1-\delta-\delta_0$.  In the small-jump regime $\kappa_k\sqrt x/\nu_L\to0$, the first term dominates and
$|\widetilde\eta_k-\eta_k|=O_{\Pp}(\kappa_k^{-2})$.
\end{theorem}
The $x^{3/2}$ term is the large-deviation component of a degree-three sub-Weibull partial sum.  For bounded diagonal directions with fixed $r$, it can be absorbed into the first term over the usual small-jump range.

\subsection{A matching lower bound with zero degree-two projection signal}
Consider the rank-one family
\begin{equation}
 f_\theta(x)=1+\theta\phi_3(x_1),\qquad |\theta|\le(2\sqrt7)^{-1}.
 \label{eq:lower-density}
\end{equation}
It is a valid density relative to $\mu$.  Every total-degree-at-most-two coefficient equals that of the uniform density, but
$A(f_\theta)-A(f_0)=\theta e_1^{\otimes3}$ and the jump size is $|\theta|$.

\begin{theorem}[Minimax lower bound]
\label{thm:lower}
There are numerical constants $c,c'>0$ such that, for $0<\kappa\le c$ and spacing large enough to contain $c/\kappa^2$ observations,
\begin{equation}
 \inf_{\widehat\eta}\sup_{P\in\cP(\kappa)}
 \E_P|\widehat\eta-\eta(P)|\ge\frac{c'}{\kappa^2},
 \label{eq:lower-bound}
\end{equation}
where $\cP(\kappa)$ contains one-change models using $f_0$ and $f_\kappa$ from \cref{eq:lower-density}.  Every segment-mean CUSUM formed from the degree-at-most-two projection tensor has zero population jump on this family.
\end{theorem}
The proof compares two change locations separated by $h\asymp\kappa^{-2}$.  Their Kullback--Leibler divergence is $O(h\kappa^2)$, so Le Cam's two-point lemma applies.  Together with \cref{thm:refinement}, this establishes minimax-order localization in the small-jump regime.

\section{Experiments}
\label{sec:experiments}

All experiments are generated by the code in \texttt{anc/code/degree3\_experiments.py}.  The random seeds, raw replication-level CSV files, threshold calibration, and plotting scripts are included in the arXiv ancillary directory.

\subsection{Pure cubic alternatives}
Let
\begin{equation}
 f_\theta(x)=1+\sum_{j=1}^d\theta_j\phi_3(x_j),
 \qquad \sqrt7\|\theta\|_1<1.
 \label{eq:simulation-density}
\end{equation}
Rejection sampling is exact because the proposal is $\mu$ and the envelope is $1+\sqrt7\|\theta\|_1$.  Orthogonality gives
\[
 \E_{f_\theta}\phi_3(X_j)=\theta_j,
 \qquad
 \E_{f_\theta}\psi_\alpha(X)=0\quad\text{for }1\le|\alpha|\le2.
\]
Thus every mean, every degree-two marginal coefficient, and every pairwise degree-two interaction is unchanged.  The degree-three jump is diagonal odeco with Frobenius norm $\|\theta^+-\theta^-\|_2$.

\subsection{Single change and dimension scaling}
For $d\in\{20,50,100,200\}$, set $n=60d$, $\eta=n/2$, $\theta^-=0$, and $\theta^+=0.34e_1$.  We use $30$ independent replications.  \LRD{} uses a cross-fitted rank-one diagonal score.  Baselines are the unregularized $\ell_2$ norm of all cubic diagonals, the exact full degree-two matrix score, all degree-one/two marginal features, a mean CUSUM, and an oracle that knows the changing cubic coordinate.

\begin{table}[t]
\centering
\small
\caption{Median absolute localization error over $30$ replications.  Runtime is the mean wall-clock time for computing all listed methods on one replication.}
\label{tab:single-results}
\begin{tabular}{@{}rrrrrrrr@{}}
\toprule
$d$ & $n$ & LR-D3 pre. & LR-D3 ref. & all cubic & degree two & mean & runtime (s)\\
\midrule
20  & 1200  & 34 & 32 & 48  & 462  & 324  & 0.009\\
50  & 3000  & 24 & 42 & 52  & 1222 & 753  & 0.021\\
100 & 6000  & 37 & 38 & 132 & 2496 & 2004 & 0.090\\
200 & 12000 & 34 & 30 & 52  & 5040 & 3417 & 0.339\\
\bottomrule
\end{tabular}
\end{table}

\begin{figure}[t]
\centering
\includegraphics[width=0.92\linewidth]{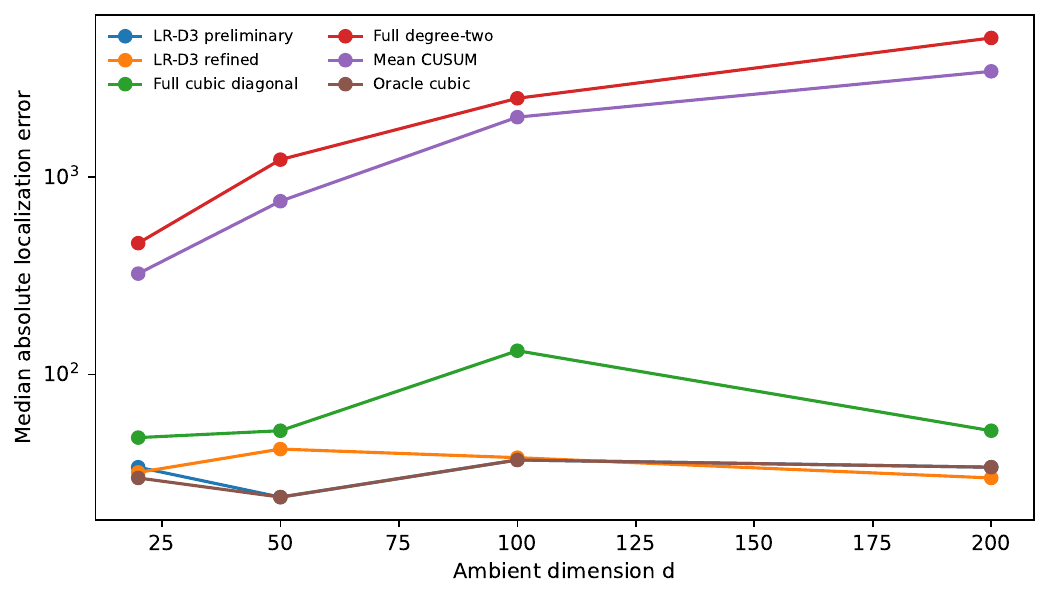}
\caption{Median single-change localization error.  The structured degree-three methods and oracle remain stable as $d$ grows, while procedures with zero population signal drift over an $O(n)$ range.  The unregularized all-cubic score pays for all $d$ coordinates.}
\label{fig:dimension}
\end{figure}

The oracle and preliminary LR-D3 results coincide at $d\ge50$ in median, indicating that the changing coordinate is reliably selected.  Refinement is a local finite-sample operation rather than a guaranteed monotone correction: it improves the median at $d=20$ and $d=200$, is comparable at $d=100$, and is less favorable at $d=50$.  In the multiple-change experiment below, the padded preliminary scan is already accurate and the two-way median safeguard leaves its aggregate error essentially unchanged.

\subsection{Population tent and full-tensor stress test}
For $d=100,n=800$, \cref{fig:score} compares the exact population score with the diagonal empirical score.  Both peak close to the true change.  The corresponding unregularized full-tensor optimization is shown in \cref{app:experiments}; at this sample size its noise is much larger, as predicted by the second term in \cref{eq:full-cusum-bound}.

\begin{figure}[t]
\centering
\includegraphics[width=0.92\linewidth]{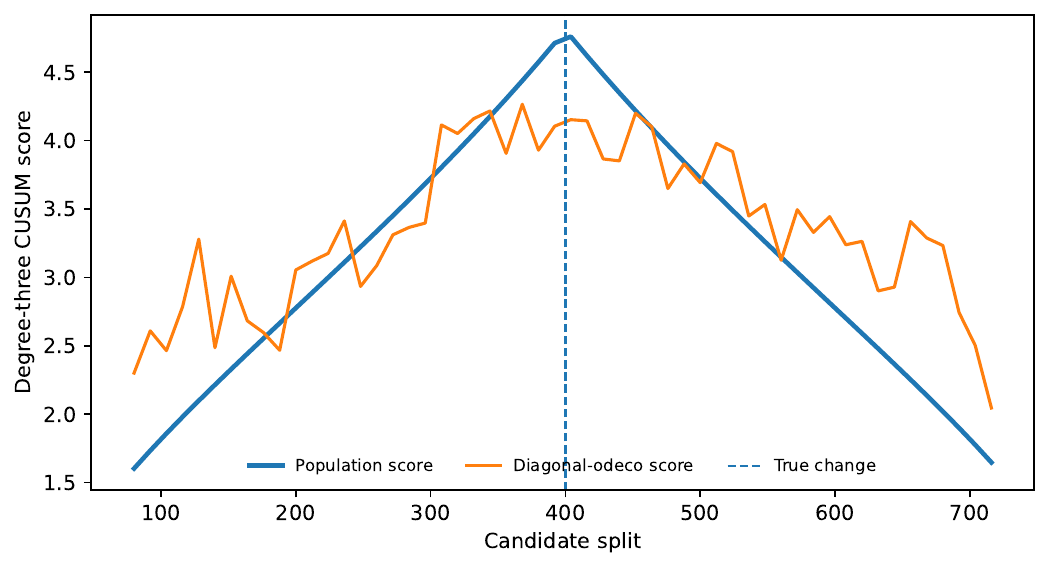}
\caption{Population and diagonal-odeco degree-three CUSUM scores for a $d=100$ pure cubic change at $400$.}
\label{fig:score}
\end{figure}

\subsection{Three changes in \texorpdfstring{$d=100$}{d=100}}
Set $n=9600$, $d=100$, and true changes $(2400,4800,7200)$.  The four segment parameters are
\[
 0,\qquad0.36e_1,\qquad0.36e_2,\qquad0.36e_3.
\]
The jump ranks are $1,2,2$.  We use base scale $h=1600$.  The threshold $6.2025$ is the empirical $97.5\%$ quantile of the maximum seeded null score over $40$ independently simulated null sequences, with every integer central split evaluated.  Across $30$ change-point replications, the estimated number of changes equals three in every run.  The median Hausdorff errors are $27$ and $28$ for the preliminary and refined estimates, respectively, while the corresponding mean errors are $43.40$ and $43.47$.  Mean runtime is $0.297$ seconds on the execution environment used for the ancillary results.  Thus the refinement is statistically comparable here rather than uniformly improving an already sharp preliminary estimate.

\begin{figure}[t]
\centering
\includegraphics[width=0.92\linewidth]{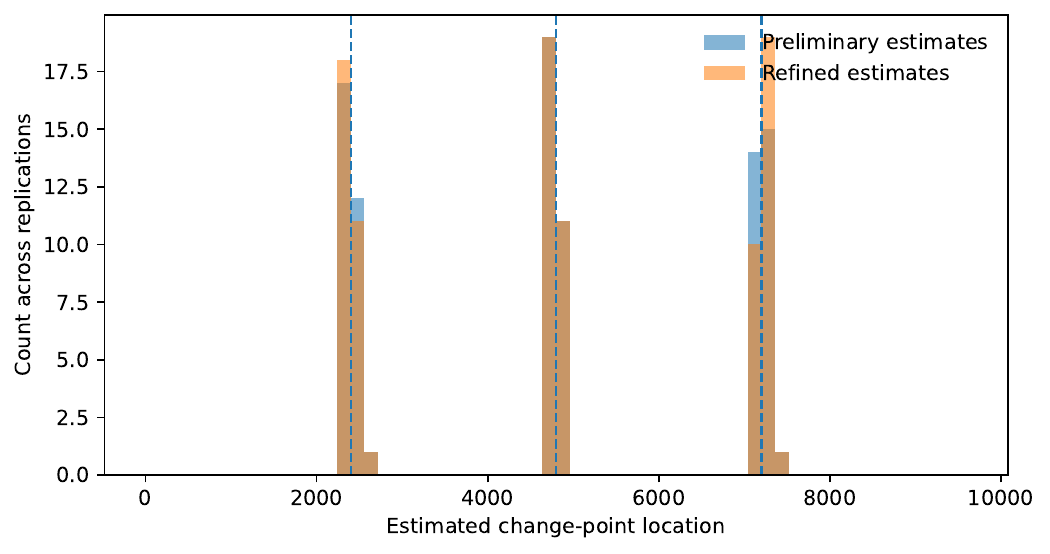}
\caption{Preliminary and refined estimates over $30$ replications of the $d=100$ three-change experiment.  Dashed lines mark the true changes.}
\label{fig:multiple}
\end{figure}

\section{Discussion and limitations}
\label{sec:discussion}

Degree-three projection detects distributional information unavailable to degree two while retaining an exact finite-dimensional geometry.  The price is genuine tensor complexity.  The full injective-norm theorem shows that a sample size merely proportional to $d$ may be insufficient for unrestricted tensor directions because the $d^{3/2}/N$ term can dominate.  Stronger structure---coordinate odeco, a known dictionary, or a small directional sketch---can reduce this complexity.  Our experiments use the exact coordinate-odeco specialization and should not be read as evidence that arbitrary order-three tensors can be estimated at the same sample size.

The method is projection-limited: if $P_3(f_1-f_0)=0$, the population signal is zero.  It is also specialized to independent observations in the main theorems.  \Cref{app:mixing} gives a geometric $\beta$-mixing extension through blocking, at the cost of logarithmic block factors.  Finally, generic tensor injective-norm optimization is nonconvex and computationally hard.  Odeco tensor power methods and our implicit Stiefel optimization are practical under favorable structure but do not guarantee a global maximizer for an arbitrary empirical tensor.  The diagonal implementation avoids this optimization entirely.

\paragraph{Reproducibility.}
The source archive contains the exact code, requirements, raw CSV results, threshold calibration, figure data, and commands used to regenerate every numerical result.  No external dataset is required.  The tensor is evaluated implicitly; the code includes a Monte Carlo check of \cref{eq:direction-isometry}.

\paragraph{Broader impact and ethics.}
The work is methodological and uses synthetic data.  Change-point systems can affect monitoring and intervention decisions; practitioners should calibrate false-alarm rates for their application and should not interpret a detected projection change as causal evidence.  The polynomial projection may also miss changes outside the retained degree.

\appendix

\section{Proofs for the degree-three representation}
\label{app:representation}

\begin{proof}[Proof of \cref{prop:isometry}]
By construction, $\{E_\alpha\}_{|\alpha|\le3}$ is an orthonormal basis of $\Sym^3(\R^p)$ and $\{\psi_\alpha\}_{|\alpha|\le3}$ is an orthonormal basis of the total-degree-three polynomial space.  Therefore
\[
 A(f)=\sum_{|\alpha|\le3}\E_f\psi_\alpha(X)E_\alpha
 =\sum_{|\alpha|\le3}\theta_\alpha(f)E_\alpha.
\]
Parseval's identity on the tensor basis and then on the polynomial basis gives
\[
 \|A(f)-A(g)\|_\F^2
 =\sum_{|\alpha|\le3}|\theta_\alpha(f)-\theta_\alpha(g)|^2
 =\|P_3(f-g)\|_{L^2(\mu)}^2.
\]
For $B=\sum b_\alpha E_\alpha$,
$\langle H_3(X),B\rangle=\sum b_\alpha\psi_\alpha(X)$, and orthonormality under $\mu$ yields
$\E_\mu\langle H_3(X),B\rangle^2=\sum b_\alpha^2=\|B\|_\F^2$.
\end{proof}

\subsection{Verification of the entry table}
For a multi-index $\alpha$, the entry assigned to each distinct permutation is
$\theta_\alpha/\sqrt{q_\alpha}$.  For example, for a quadratic univariate term,
$\theta_{2e_j}=\E_f\phi_2(X_j)=\sqrt5\E_fp_2(X_j)=b_j/\sqrt5$ and $q_{2e_j}=3$, giving $b_j/\sqrt{15}$.  For a mixed cubic term,
$\theta_{2e_i+e_j}=\sqrt{15}\E_fp_2(X_i)X_j=\eta_{ij}/\sqrt{15}$ and $q=3$, giving $\eta_{ij}/\sqrt{45}=\eta_{ij}/(3\sqrt5)$.  The other rows are identical calculations.

\subsection{Derivation of the implicit contraction}
Write $u=(a,b)$ and expand
\[
 H_3(x)[u,u,u]=\sum_{|\alpha|\le3}\psi_\alpha(x)E_\alpha[u,u,u].
\]
For each $\alpha$, $E_\alpha[u,u,u]=\sqrt{q_\alpha}a^{3-|\alpha|}\prod_jb_j^{\alpha_j}$.  Grouping by the seven multi-index types gives
\begin{align*}
&a^3+\sqrt3a^2\sum_jb_j\phi_1(x_j)
+\sqrt3a\sum_jb_j^2\phi_2(x_j)
+\sqrt6a\sum_{i<j}b_ib_j\phi_1(x_i)\phi_1(x_j)\\
&\quad+\sum_jb_j^3\phi_3(x_j)
+\sqrt3\sum_{i\ne j}b_i^2b_j\phi_2(x_i)\phi_1(x_j)
+\sqrt6\sum_{i<j<k}b_ib_jb_k\phi_1(x_i)\phi_1(x_j)\phi_1(x_k).
\end{align*}
The pair and triple sums reduce to the elementary symmetric-polynomial identities in \cref{eq:implicit-contraction}.

\section{Tensor score properties}
\label{app:score}

\begin{proof}[Proof of \cref{prop:sod}]
Let $D=\sum_{k=1}^{r_0}\lambda_kv_k^{\otimes3}$ with $r_0\le r$.  Choosing a frame that contains $v_1,\ldots,v_{r_0}$ gives $S_r(D)\ge(\sum_k\lambda_k^2)^{1/2}=\|D\|_\F$.  Conversely, for any orthonormal frame $u_1,\ldots,u_r$, put $a_{kj}=\langle v_k,u_j\rangle$.  Cauchy--Schwarz gives
\[
 \left(\sum_k\lambda_ka_{kj}^3\right)^2
 \le\left(\sum_k\lambda_k^2a_{kj}^2\right)
 \left(\sum_ka_{kj}^4\right)
 \le\sum_k\lambda_k^2a_{kj}^2.
\]
Summing over $j$ and using Bessel's inequality, $\sum_ja_{kj}^2\le1$, yields
$\sum_jD[u_j,u_j,u_j]^2\le\sum_k\lambda_k^2$.

For the perturbation inequality, fix any frame $U$.  The Euclidean norm of the vector with coordinates $E[u_j,u_j,u_j]$ is at most $\sqrt r\|E\|_\inj$.  The reverse triangle inequality gives
\[
 \left|\|\Phi_U(T+E)\|_2-\|\Phi_U(T)\|_2\right|
 \le\sqrt r\|E\|_\inj.
\]
Taking suprema and exchanging $T$ and $T+E$ proves the claim.
\end{proof}

\begin{remark}[Approximate odeco jumps]
\label{rem:approx}
If $D=D_{\mathrm{od}}+R$ with $D_{\mathrm{od}}$ odeco of rank at most $r$, then
$|S_r(D)-\|D_{\mathrm{od}}\|_\F|\le\sqrt r\|R\|_\inj\le\sqrt r\|R\|_\F$.
All theorems remain valid after adding the corresponding approximation term to $\eps_m$ and replacing $\kappa$ by $\|D_{\mathrm{od}}\|_\F$.
\end{remark}

\section{Proof of the tensor deviation inequality}
\label{app:concentration}

\subsection{Polynomial-chaos moment input}
We derive the required moment growth from the product-ensemble hypercontractivity of \citet{mossel2010noise}.

\begin{lemma}[Degree-three product-polynomial hypercontractivity]
\label{lem:hypercontractive}
There is a numerical $C_3$ such that every polynomial $Q$ of total degree at most three in independent uniform $[-1,1]$ coordinates satisfies, for $q\ge2$,
\begin{equation}
 \|Q\|_{L^q(\mu)}\le C_3q^{3/2}\|Q\|_{L^2(\mu)}.
 \label{eq:hypercontractive}
\end{equation}
\end{lemma}
\begin{proof}
For one coordinate $U\sim\operatorname{Unif}[-1,1]$, regard
$\mathcal X=(1,\phi_1(U),\phi_2(U),\phi_3(U))$ as an orthonormal ensemble.  Any centered linear combination
$G=\sum_{k=1}^3a_k\phi_k(U)$ satisfies $\|G\|_2=\|a\|_2$ and
\begin{equation}
 \sum_{k=1}^3\phi_k(x)^2\le15,
 \qquad
 15-\sum_{k=1}^3\phi_k(x)^2
 =\frac54(1-x^2)(35x^4+2x^2+11)\ge0.
 \label{eq:legendre-envelope}
\end{equation}
Hence $\|G\|_q\le\sqrt{15}\|G\|_2$.  Proposition~3.16 of \citet{mossel2010noise} implies that every such $G$ is $(2,q,\eta_q)$-hypercontractive with
$\eta_q\ge\{2\sqrt{15}\sqrt{q-1}\}^{-1}$.  Their Proposition~3.11 tensorizes this constant over independent coordinate ensembles.  Finally, every total-degree-at-most-three Legendre expansion is an ensemble-multilinear polynomial of ensemble degree at most three, so their Proposition~3.12 gives
\[
 \|Q\|_q\le\eta_q^{-3}\|Q\|_2
 \le(2\sqrt{15})^3(q-1)^{3/2}\|Q\|_2.
\]
Absorb the fixed constant and use $q-1\le q$.
\end{proof}

\begin{proof}[Proof of \cref{lem:directional-chaos}]
Let $Q_B(x)=\langle H_3(x),B\rangle$.  By \cref{eq:direction-isometry}, $\|Q_B\|_{L^2(\mu)}=\|B\|_\F$.  Since $f\le L$,
\[
 \|Q_B\|_{L^q(f)}
 \le L^{1/q}\|Q_B\|_{L^q(\mu)}
 \le C_3\max\{1,\sqrt L\}q^{3/2}\|B\|_\F.
\]
Centering changes the $L^q$ norm by at most a factor two.  The moment characterization of $\psi_{2/3}$ variables completes the proof.
\end{proof}

\subsection{Weighted scalar concentration}
The generalized Bernstein--Orlicz inequalities of \citet{kuchibhotla2022beyond}, with the tight refinement of \citet{bong2023tight}, imply the following form: if independent centered random variables satisfy $\|Y_i\|_{\psi_{2/3}}\le K$, then for $x\ge1$,
\begin{equation}
 \Pp\left(\left|\sum_iw_iY_i\right|>
 CK\{\sqrt x\|w\|_2+x^{3/2}\|w\|_\infty\}\right)\le2e^{-x}.
 \label{eq:scalar-subweibull}
\end{equation}
The two terms are both necessary for general $\psi_{2/3}$ summands.

\begin{proof}[Proof of \cref{thm:tensor-bernstein}]
Let $\cN$ be a $1/8$-net of $\mathbb S^{p-1}$ with $|\cN|\le17^p$.  For a fixed triple $(u,v,z)\in\cN^3$, let
$B=\Sym(u\otimes v\otimes z)$.  Since $H_3(X_i)$ is symmetric,
$\langle Z_i,u\otimes v\otimes z\rangle=\langle Z_i,B\rangle$, and $\|B\|_\F\le1$.  By \cref{lem:directional-chaos} and \cref{eq:scalar-subweibull}, with $x=\rho(p,\delta)$,
\[
 \left|\sum_iw_iZ_i[u,v,z]\right|
 \le C\nu_L\{\sqrt\rho\|w\|_2+\rho^{3/2}\|w\|_\infty\}
\]
simultaneously over the net triples with probability at least $1-\delta$.

Let $T=\sum_iw_iZ_i$ and $M=\|T\|_\inj$.  For arbitrary unit $u,v,z$, choose net points $u_0,v_0,z_0$.  The telescoping identity
\[
 T[u,v,z]-T[u_0,v_0,z_0]
 =T[u-u_0,v,z]+T[u_0,v-v_0,z]+T[u_0,v_0,z-z_0]
\]
shows that the absolute difference is at most $3M/8$.  Therefore
$M\le(1-3/8)^{-1}\max_{\cN^3}|T[u_0,v_0,z_0]|$, which proves \cref{eq:tensor-bernstein} after changing the constant.
\end{proof}

\begin{proof}[Proof of \cref{cor:uniform-cusum}]
For each CUSUM, \cref{eq:cusum-weights} gives $\|w\|_2=1$.  Balance gives
$\|w\|_\infty\le(\alpha(e-s))^{-1/2}$.  Apply \cref{thm:tensor-bernstein} with failure probability $\delta/M$ and then \cref{prop:sod}.
\end{proof}

\begin{proof}[Proof of \cref{cor:diagonal}]
Each centered coordinate $\phi_3(X_{ij})-\E\phi_3(X_{ij})$ is bounded in absolute value by $2\sqrt7$.  For deterministic CUSUM weights with squared sum one, Hoeffding's inequality gives
\[
 \Pp\left(\left|\sum_iw_i\{\phi_3(X_{ij})-\E\phi_3(X_{ij})\}\right|>x\right)
 \le2\exp(-cx^2).
\]
A union bound over $dM$ coordinate-candidate pairs controls the vector $\ell_\infty$ error.  The top-$r$ Euclidean norm is $\sqrt r$-Lipschitz with respect to $\ell_\infty$.
\end{proof}

\section{Population geometry and the one-change theorem}
\label{app:single}

\subsection{Population CUSUM calculation}
For $t<\eta$, the left mean is $A_0$, whereas the right mean is
$\{(\eta-t)A_0+(e-\eta)A_1\}/(e-t)$.  Their difference is $(e-\eta)D/(e-t)$, which gives the first branch of \cref{eq:tent-coefficient}.  The $t>\eta$ branch is symmetric.  Squaring the coefficient and subtracting its value at $t$ from its value at $\eta$ gives \cref{eq:exact-margin}; for $t<\eta$,
\begin{align*}
&a_\eta(\eta)^2-a_\eta(t)^2\\
&=\frac{(\eta-s)(e-\eta)}{e-s}
-\frac{(e-\eta)^2(t-s)}{(e-s)(e-t)}
=\frac{(e-\eta)(\eta-t)}{e-t}.
\end{align*}

\begin{proof}[Proof of \cref{thm:single}]
Let $g(t)=S_r(C_{s,e}^t)$ and $\widehat g(t)$ be the empirical score.  Uniform error and empirical maximality imply
\[
 g(\eta)-g(\widehat\eta)\le2\eps_m.
\]
If $t<\eta$, balance and \cref{eq:exact-margin} give
$g(\eta)^2-g(t)^2\ge\alpha\kappa^2(\eta-t)$.  Also $g(t)\le g(\eta)$ and
$2g(\eta)\le\kappa\sqrt m$, because $(\eta-s)(e-\eta)/(e-s)\le m/4$.  Hence
\[
 g(\eta)-g(t)=\frac{g(\eta)^2-g(t)^2}{g(\eta)+g(t)}
 \ge\frac{\alpha\kappa|t-\eta|}{\sqrt m}.
\]
The same inequality holds for $t>\eta$, proving \cref{eq:single-rate}.  On a null interval $g(t)=0$ for every candidate and the empirical maximum is at most $\eps_m$.  On the one-change interval,
$g(\eta)\ge\alpha\kappa\sqrt m$, so the empirical maximum is at least this value minus $\eps_m$.  This proves \cref{eq:threshold-window}.
\end{proof}

\section{Full induction proof for multiple changes}
\label{app:multiple}

We give the deterministic proof on the uniform deviation event.  The constants below are numerical and can be chosen in sequence; their exact values are not optimized.

\subsection{Seed geometry}
\begin{lemma}[Balanced isolating seed]
\label{lem:seed-geometry}
Let $\eta$ lie in a recursive segment $(a,c]$ and suppose
$\min\{\eta-a,c-\eta\}\ge3h/4$.  Then the half-overlap base layer contains an interval $I=(s,e]\subset(a,c]$ of length $h$ such that
\begin{equation}
 \min\{\eta-s,e-\eta\}\ge h/4.
 \label{eq:seed-balance}
\end{equation}
\end{lemma}
\begin{proof}
The length-$h$ intervals start $h/2$ apart.  The two intervals whose overlap contains $\eta$ place $\eta$ at complementary offsets summing to $h/2$; one has both endpoint distances at least $h/4$.  The assumed distance from $(a,c]$ ensures containment.
\end{proof}

\subsection{Boundary contamination}
\begin{lemma}[A scanned interval with no undetected change]
\label{lem:contamination}
Suppose a scanned interval contains no undetected true change.  Relative to its constant interior mean, assume at most $R_L$ observations inherited from the left recursive boundary have mean shift $D_L$ and at most $R_R$ observations inherited from the right boundary have mean shift $D_R$.  On any seed interval of length $m\ge h$ with central candidates,
\begin{equation}
 \sup_t S_r(C_{s,e}^t)
 \le \frac{2}{\sqrt h}
 \{R_L\|D_L\|_\F+R_R\|D_R\|_\F\}
 =\frac{2}{\sqrt h}(R_L\kappa_L+R_R\kappa_R).
 \label{eq:contamination-bound}
\end{equation}
The same bound holds for the diagonal score.
\end{lemma}
\begin{proof}
CUSUMs are invariant to adding a constant tensor mean.  Each nonzero edge mean is multiplied by a weight of absolute value at most $2/\sqrt m$ because candidates lie in the central half.  Therefore
\[
 \|C_{s,e}^t\|_\F
 \le\frac{2}{\sqrt h}
 \{R_L\|D_L\|_\F+R_R\|D_R\|_\F\}.
\]
For every tensor $T$, the tensors $u_1^{\otimes3},\ldots,u_r^{\otimes3}$ generated by an orthonormal frame are themselves Frobenius-orthonormal.  Bessel's inequality therefore gives $S_r(T)\le\|T\|_\F$.  The top-$r$ diagonal score is also at most the Euclidean, hence Frobenius, norm.  This proves the claim.
\end{proof}

\begin{proof}[Proof of \cref{thm:multiple}]
We induct on the number of detections already made.  The induction hypothesis is:
(i) every reported estimator is paired with a distinct true change and obeys \cref{eq:multiple-rate};
(ii) every remaining true change is in exactly one current recursive segment and is at distance at least $3h/4$ from that segment's boundaries; and
(iii) any scanned subinterval containing no remaining change has only the inherited edge contaminations described in \cref{lem:contamination}.
The hypothesis is true initially because true changes and global boundaries are separated by at least $\Delta$, while $h\le3\Delta/4$.

\paragraph{Step 1: intervals with no undetected change are inactive.}
For any seeded interval containing no undetected change, \cref{lem:contamination} and the induction error bounds give a population score bounded by a fixed multiple of
$\max_k\kappa_kR_k/\sqrt h$, hence by a fixed multiple of $\eps_h$ through \cref{eq:Rk}.  Adding empirical error $\eps_h$ remains below $\tau$ when $C_0$ is sufficiently large.  Thus neither a terminal recursive segment nor a change-free seed inside a nonterminal segment can create a duplicate or false detection.

\paragraph{Step 2: a segment with an undetected change is active.}
Let $\eta_k$ be any undetected change in the segment.  By the induction boundary margin and \cref{lem:seed-geometry}, there is a length-$h$ isolating seed $I$ satisfying \cref{eq:seed-balance}.  It contains no other true change because $h<\Delta$.  Its population score at $\eta_k$ is at least $\kappa_k\sqrt h/4$.  After subtracting $\eps_h$, this exceeds $\tau$ by the upper threshold and signal conditions.  Thus the active set is nonempty.

\paragraph{Step 3: the chosen seed contains one change and the padded scan localizes it.}
The seeded family has no interval shorter than $h$, while Step~2 supplies an active interval of length $h$.  Hence the shortest active seed $I^*=(u,v]$ has length exactly $h$.  Since $h<\Delta$, it contains at most one true change.  It cannot contain zero true changes by Step~1, so it contains exactly one, say $\eta_k$.

Set $g=h/8$, $u^+=\max\{s,u-g\}$, and $v^+=\min\{e,v+g\}$ as in \cref{eq:padded-recentering}.  If an endpoint is not clipped, the padding places $\eta_k$ at distance at least $g$ from it; if it is clipped at a recursive boundary, induction hypothesis (ii) gives the stronger distance $3h/4$.  Consequently
\[
 \min\{\eta_k-u^+,v^+-\eta_k\}\ge g,
 \qquad v^+-u^+\le h+2g=5h/4.
\]
The padded interval contains no other change because $5h/4\le15\Delta/16<\Delta$, and its balance is at least
$g/(5h/4)=1/10$.  Its candidate range contains $\eta_k$.  Applying \cref{thm:single} on the uniform deviation event gives
$|\widehat\eta_k-\eta_k|\le R_k$ after choosing $C_{\mathrm{loc}}$ to absorb the fixed constants.

\paragraph{Step 4: preservation of the induction hypothesis.}
Since $R_k\le h/8<\Delta/2$, splitting at $\widehat\eta_k$ assigns every other true change to the correct child.  The nearest remaining change is at distance at least
$\Delta-R_k\ge\Delta-h/8\ge3h/4$ from the new boundary, where the last inequality follows from $h\le3\Delta/4$.  Thus balanced isolating seeds continue to exist.  A child with no remaining change contains at most $R_k$ observations from the wrong regime at its new edge, exactly the contamination form in \cref{lem:contamination}.  This completes the induction.

After $K$ steps, no undetected change remains and Step~1 terminates all recursive calls.  Hence exactly $K$ estimates are returned.
\end{proof}

\section{Proofs for cross-fitted refinement}
\label{app:refinement}

\begin{proof}[Proof of \cref{lem:alignment}]
Let $\widehat C=aD+E$, let $\widehat U$ maximize the frame score, and let $\widehat V$ be defined by \cref{eq:direction}.  By construction,
\[
 \langle\widehat C,\widehat V\rangle=S_r(\widehat C).
\]
Moreover, writing $\widehat V=\sum_jq_j\widehat u_j^{\otimes3}$ with $\sum_jq_j^2=1$,
\[
 |\langle E,\widehat V\rangle|
 \le\sum_j|q_j|\,|E[\widehat u_j^3]|
 \le\sqrt r\|E\|_\inj.
\]
Therefore
\begin{align*}
 a\langle D,\widehat V\rangle
 &=S_r(\widehat C)-\langle E,\widehat V\rangle\\
 &\ge S_r(aD)-2\sqrt r\|E\|_\inj
 \ge a\kappa-2\sqrt r\lambda.
\end{align*}
The assumed inequality yields \cref{eq:alignment}.
\end{proof}

\subsection{A scalar localization lemma}
\begin{lemma}[Sub-Weibull least-squares change localization]
\label{lem:scalar-localization}
Let independent scalar observations have means $\mu_L$ before $\eta$ and $\mu_R$ after $\eta$, with $\delta_\mu=|\mu_R-\mu_L|>0$.  Suppose centered errors have $\psi_{2/3}$ norm at most $\nu$, and independent anchor estimates satisfy
$|\widehat\mu_L-\mu_L|\vee|\widehat\mu_R-\mu_R|\le\delta_\mu/8$.  Let $\widehat\eta$ minimize the residual sum of squares using these two fixed means over a search interval containing $\eta$.  Then, with probability at least $1-2e^{-x}$,
\begin{equation}
 |\widehat\eta-\eta|
 \le C\left\{\frac{\nu^2x}{\delta_\mu^2}
 +\frac{\nu x^{3/2}}{\delta_\mu}+1\right\}.
 \label{eq:scalar-rate}
\end{equation}
\end{lemma}
\begin{proof}
Assume without loss of generality that $\mu_R>\mu_L$ and consider $t>\eta$; the other side is identical.  Put $h=t-\eta$, $\widehat\delta=\widehat\mu_R-\widehat\mu_L$, and write the post-change observations as $Z_i=\mu_R+\xi_i$.  The exact objective increment is
\begin{align*}
 Q(t)-Q(\eta)
 &=\sum_{i=\eta+1}^{t}\{(Z_i-\widehat\mu_L)^2-(Z_i-\widehat\mu_R)^2\}\\
 &=h\widehat\delta\{2\mu_R-\widehat\mu_L-\widehat\mu_R\}
   +2\widehat\delta\sum_{i=\eta+1}^{t}\xi_i.
\end{align*}
On the anchor event, both factors in the deterministic product lie between
$3\delta_\mu/4$ and $5\delta_\mu/4$.  Hence
\begin{equation}
 Q(t)-Q(\eta)
 \ge c h\delta_\mu^2-C\delta_\mu
 \left|\sum_{i=\eta+1}^{t}\xi_i\right|.
 \label{eq:rss-drift}
\end{equation}

We next record a maximal partial-sum consequence of the same sub-Weibull inequality used in \cref{eq:scalar-subweibull}.  For every $q\ge2$ and $m\ge1$,
\[
 \left\|\sum_{i=1}^{m}\xi_i\right\|_q
 \le C\nu\{\sqrt{mq}+q^{3/2}\}.
\]
The partial sums form a martingale, so Doob's $L^q$ inequality and Markov's inequality imply
\begin{equation}
 \Pp\left(
  \max_{1\le k\le m}\left|\sum_{i=1}^k\xi_i\right|
  >C\nu\{\sqrt{mq}+q^{3/2}\}
 \right)\le e^{-q}
 \label{eq:maximal-subweibull}
\end{equation}
(after enlarging $C$).

Let
$M=C_0\{\nu^2x/\delta_\mu^2+\nu x^{3/2}/\delta_\mu+1\}$.
Peel the candidate distances into shells $2^jM\le h<2^{j+1}M$ and apply
\cref{eq:maximal-subweibull} with $m=\lceil2^{j+1}M\rceil$ and
$q_j=x+j+2$.  The elementary bounds
$(x+j+2)/2^j\le Cx$ and
$(x+j+2)^{3/2}/2^j\le Cx^{3/2}$ for $x\ge1$ show, by the definition of $M$ and a sufficiently large $C_0$, that throughout shell $j$,
\[
 C\nu\{\sqrt{2^{j+1}Mq_j}+q_j^{3/2}\}
 < c\,2^jM\delta_\mu.
\]
Thus \cref{eq:rss-drift} is strictly positive for every $h\ge M$ outside an event of probability
$\sum_{j\ge0}e^{-q_j}\le Ce^{-x}$.  Repeating on the left and changing constants proves \cref{eq:scalar-rate} with probability at least $1-2e^{-x}$.
\end{proof}

\begin{proof}[Proof of \cref{thm:refinement}]
Condition on the pilot and outer-anchor observations.  Then $\widehat V$ and the two anchor means are fixed, while the central held-out observations remain independent.  By \cref{lem:alignment}, the scalar mean jump is
$\delta_\mu=|\langle D_k,\widehat V\rangle|\ge\kappa_k/2$.  \Cref{lem:directional-chaos} gives held-out centered $\psi_{2/3}$ norm at most $\nu_L$.  Apply \cref{lem:scalar-localization} with $x=\log(8K/\delta)$ to each fold and each change and take a union bound.  If both fold-specific estimates satisfy the stated radius, their median with the preliminary estimate also satisfies it.  The theorem's anchor-error assumption is sufficient for the scalar lemma because $\delta_\mu\ge\kappa_k/2$, so an error of at most $\kappa_k/16$ is no larger than $\delta_\mu/8$.
\end{proof}

\section{Proof of the minimax lower bound}
\label{app:lower}

\begin{proof}[Proof of \cref{thm:lower}]
The density in \cref{eq:lower-density} is nonnegative because
$|\theta\phi_3(x_1)|\le1/2$, and it integrates to one because $\E_\mu\phi_3=0$.  Orthogonality proves that every coefficient of total degree at most two is unchanged and that the tensor jump is $\theta e_1^{\otimes3}$.

Choose two models whose change locations differ by $h$.  The joint laws differ only on those $h$ observations.  For $|z|\le1/2$, $(1+z)\log(1+z)\le z+2z^2$.  Therefore
\begin{align*}
 \mathrm{KL}(f_\kappa\,\|\,f_0)
 &=\int(1+\kappa\phi_3)\log(1+\kappa\phi_3)\,d\mu\\
 &\le\kappa\int\phi_3\,d\mu+2\kappa^2\int\phi_3^2\,d\mu
 =2\kappa^2.
\end{align*}
Thus the product-law divergence is at most $2h\kappa^2$.  Take
$h=\lfloor c_0/\kappa^2\rfloor$ with $c_0$ sufficiently small.  Pinsker's inequality and Le Cam's two-point lemma \citep{tsybakov2009introduction} give
\[
 \inf_{\widehat\eta}\max_{j=0,1}
 \E_j|\widehat\eta-\eta_j|
 \ge\frac h4\{1-\mathrm{TV}(P_0,P_1)\}
 \ge c'\kappa^{-2}.
\]
\end{proof}

\section{Geometrically beta-mixing extension}
\label{app:mixing}

The main paper assumes independence to isolate the tensor issue.  A direct extension is available under geometric absolute regularity.  Let $Z_i=H_3(X_i)-\E H_3(X_i)$ and assume
\[
 \beta(q)\le C_\beta e^{-c_\beta q}.
\]
Partition a CUSUM interval into consecutive blocks of length
$q\asymp c_\beta^{-1}\log(2m/\delta)$, separate odd and even blocks, and apply Berbee coupling to each family.  The total coupling failure is at most $\delta$.  Within a block $J$, the quasi-triangle inequality for $\psi_{2/3}$ norms bounds a weighted contraction by $C\nu_La_J$, where $a_J=\sum_{i\in J}|w_i|$.  Applying the independent weighted tensor theorem to the coupled blocks gives the following conservative corollary.

\begin{corollary}[Geometric mixing, block form]
\label{cor:mixing}
Under the preceding assumptions, with probability at least $1-2\delta$,
\begin{equation}
 \left\|\sum_iw_iZ_i\right\|_\inj
 \le C\nu_L\left\{\sqrt{q\rho}\|w\|_2+q\rho^{3/2}\|w\|_\infty\right\},
 \label{eq:mixing-bound}
\end{equation}
where $\rho=3p\log17+\log(C/\delta)$.  For balanced CUSUMs, the independent rate is inflated by at most $\sqrt q$ in the Gaussian term and $q$ in the third-order term.
\end{corollary}
\begin{proof}
For each parity of blocks, Berbee coupling constructs independent block copies with the same block marginals; the failure probability is at most the number of blocks times $\beta(q)$, hence at most $\delta$ by the choice of $q$.  For a fixed tensor direction $B$ with $\|B\|_\F\le1$, the contraction of block $J$ is
$Y_J=\sum_{i\in J}w_i\langle Z_i,B\rangle$.  The quasi-triangle inequality for $\psi_{2/3}$ gives
$\|Y_J\|_{\psi_{2/3}}\le C\nu_La_J$, where
$a_J=\sum_{i\in J}|w_i|$.  Apply the independent scalar inequality to the normalized block variables with deterministic weights $a_J$.  Cauchy--Schwarz within each block yields
\[
 \left(\sum_Ja_J^2\right)^{1/2}\le\sqrt q\|w\|_2,
 \qquad \max_Ja_J\le q\|w\|_\infty.
\]
The three-sphere net argument from \cref{thm:tensor-bernstein} controls all directions, and summing the odd- and even-block bounds changes only the universal constant.  Adding the coupling failure proves the claim.
\end{proof}
This block result is not claimed to be sharp in the mixing parameters.  It shows that the induction and localization arguments are unchanged once an appropriate uniform tensor deviation event is available.

\section{Additional computational and experimental details}
\label{app:experiments}

\subsection{Full-tensor stress curve}
\begin{figure}[H]
\centering
\includegraphics[width=0.92\linewidth]{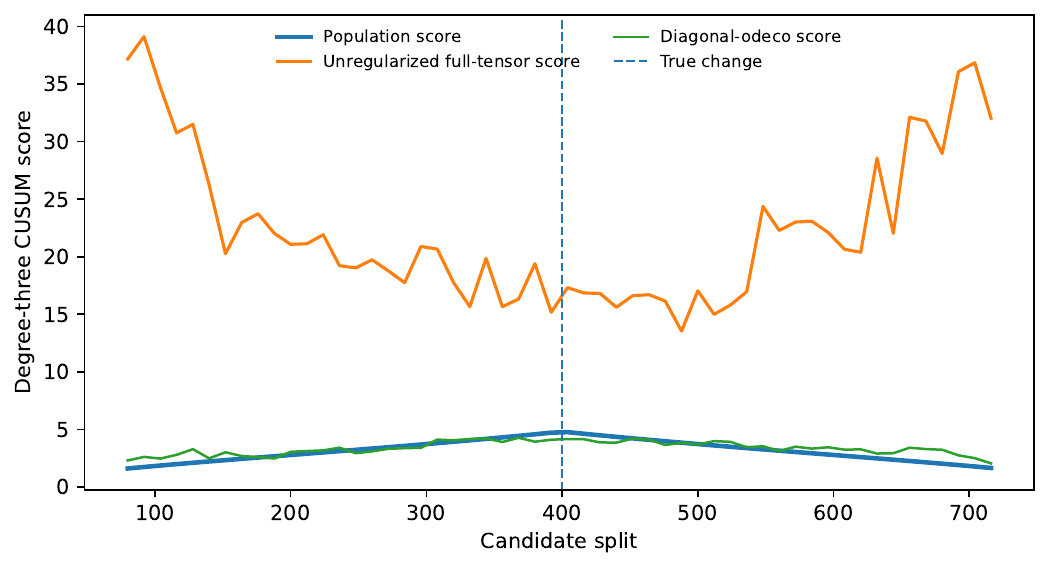}
\caption{At $d=100,n=800$, the unregularized implicit full-tensor score is dominated by tensor noise, while the diagonal-odeco score tracks the population tent.  This behavior is consistent with the $d^{3/2}/\sqrt n$ CUSUM term in \cref{eq:full-cusum-bound}.}
\label{fig:full-stress}
\end{figure}

\subsection{Monte Carlo isometry check}
For random unit vectors $u,v\in\R^9$ and $250{,}000$ uniform samples in $d=8$, the code returned
\[
 \widehat\E H_3(X)[u^3]^2=1.0110,\qquad
 \widehat\E H_3(X)[v^3]^2=1.0075,
\]
and
\[
 \widehat\E H_3(X)[u^3]H_3(X)[v^3]=0.04091,
 \qquad \langle u,v\rangle^3=0.03950.
\]
This is a numerical check, not part of the proof.

\subsection{Complete single-change summaries}
\begin{table}[H]
\centering
\scriptsize
\caption{Median, mean, and $90\%$ quantile of absolute localization error over $30$ replications.}
\label{tab:full-single}
\begin{tabular}{@{}rrlrrr@{}}
\toprule
$d$ & $n$ & Method & Median & Mean & $q_{0.9}$\\
\midrule
20&1200&LR-D3 preliminary&34&88.33&268.4\\
20&1200&LR-D3 refined&32&81.33&288.6\\
20&1200&All cubic diagonal&48&119.87&303.6\\
20&1200&Full degree two&462&420.80&504.0\\
20&1200&Mean CUSUM&324&319.20&493.2\\
20&1200&Oracle cubic&30&70.27&251.6\\
\addlinespace
50&3000&LR-D3 preliminary&24&48.87&108.4\\
50&3000&LR-D3 refined&42&64.93&122.8\\
50&3000&All cubic diagonal&52&106.87&258.4\\
50&3000&Full degree two&1222&1181.60&1260.0\\
50&3000&Mean CUSUM&753&739.00&1236.6\\
50&3000&Oracle cubic&24&48.87&108.4\\
\addlinespace
100&6000&LR-D3 preliminary&37&58.20&140.6\\
100&6000&LR-D3 refined&38&63.47&135.2\\
100&6000&All cubic diagonal&132&228.20&315.0\\
100&6000&Full degree two&2496&2481.20&2520.0\\
100&6000&Mean CUSUM&2004&1620.53&2520.0\\
100&6000&Oracle cubic&37&58.20&140.6\\
\addlinespace
200&12000&LR-D3 preliminary&34&64.07&154.0\\
200&12000&LR-D3 refined&30&79.80&217.8\\
200&12000&All cubic diagonal&52&113.20&190.0\\
200&12000&Full degree two&5040&5027.00&5040.0\\
200&12000&Mean CUSUM&3417&3095.60&4964.4\\
200&12000&Oracle cubic&34&64.07&154.0\\
\bottomrule
\end{tabular}
\end{table}

\subsection{Implementation details}
The large Monte Carlo loops exploit the exact diagonal-odeco model.  For each coordinate, a prefix sum of $\phi_3(X_{ij})$ is computed once.  A candidate CUSUM is then a vector operation followed by selection of the largest $r$ absolute coordinates.  The full degree-two baseline is evaluated by compact prefix matrices only at candidate locations.  The representative full-tensor curve uses PyTorch, the implicit contraction \cref{eq:implicit-contraction}, Adam iterations, and QR projection onto the Stiefel manifold; it is not used in the high-dimensional replication loops.

\end{document}